\def\year{2021}\relax
\documentclass[letterpaper]{article} 
\usepackage{aaai21}  
\usepackage{times}  
\usepackage{helvet} 
\usepackage{courier}  
\usepackage[hyphens]{url}  
\usepackage{graphicx} 
\usepackage{natbib}  
\usepackage{caption} 
\title{Compressed Deep Networks: Goodbye SVD, Hello Robust Low-Rank Approximation}
\author{Murad Tukan\textsuperscript{\rm 1}\thanks{These authors contributed equally to this work.}, Alaa Maalouf\textsuperscript{\rm 1}\printfnsymbol{1}, Matan Weksler\textsuperscript{\rm 2}\printfnsymbol{1}, Dan Feldman\textsuperscript{\rm 1}\\
}
\affiliations{

    \textsuperscript{\rm 1} The Robotics and Big Data Lab, Department of Computer Science, University of Haifa, Haifa, Israel\\
    \textsuperscript{\rm 2} Samsung Research Israel \\
    muradtuk@gmail.com, alaamalouf12@gmail.com,
    m.weksler@samsung.com,
    dannyf.post@gmail.com

}

\usepackage{papercommands}
\usepackage[inline]{enumitem}
\usepackage[ruled,vlined,linesnumbered,noend]{algorithm2e}
\usepackage{subcaption}
\usepackage{thmtools}
\usepackage{thm-restate}
\usepackage{multirow}

\usepackage{xcolor}
\usepackage{enumitem}
\usepackage{comment}
\usepackage{todonotes}
\usepackage{array}

\newcolumntype{P}[1]{>{\centering\arraybackslash}p{#1}}

\makeatletter
\newcommand{\printfnsymbol}[1]{%
  \textsuperscript{\@fnsymbol{#1}}%
}
\makeatother

\nocopyright

\begin{document}

\maketitle

\begin{abstract}
A common technique for compressing a neural network is to compute the $k$-rank $\ell_2$ approximation $A_{k,2}$ of the matrix $A\in\mathbb{R}^{n\times d}$ that corresponds to a fully connected layer (or embedding layer).
Here, $d$ is the number of the neurons in the layer, $n$ is the number in the next one, and $A_{k,2}$ can be stored in $O((n+d)k)$ memory instead of $O(nd)$.

This $\ell_2$-approximation minimizes the sum over every entry to the power of $p=2$ in the matrix $A - A_{k,2}$, among every matrix $A_{k,2}\in\mathbb{R}^{n\times d}$ whose rank is $k$. While it can be computed efficiently via SVD, the $\ell_2$-approximation is known to be very sensitive to outliers (``far-away" rows).
Hence, machine learning uses e.g. Lasso Regression, $\ell_1$-regularization, and $\ell_1$-SVM that use the $\ell_1$-norm.

This paper suggests to replace the $k$-rank $\ell_2$ approximation by $\ell_p$, for $p\in [1,2]$.
We then provide practical and provable approximation algorithms to compute it for any $p\geq1$, based on modern techniques in computational geometry.

Extensive experimental results on the GLUE benchmark for compressing BERT, DistilBERT, XLNet, and RoBERTa confirm this theoretical advantage. For example, our approach achieves $28\%$ compression of RoBERTa's embedding layer with only $0.63\%$ additive drop in the accuracy (without fine-tuning) in average over all tasks in GLUE, compared to $11\%$ drop using the existing $\ell_2$-approximation. Open code is provided for reproducing and extending our results.
\end{abstract}

\section{Background} \label{sec:Background}
We first give background regarding network compression, the approach of low-rank approximation, and the limitation of the common $\ell_2$-approximation.

\paragraph{Network Compression.}
Deep Learning revolutionized Machine Learning by improving the accuracy by dozens of percents for fundamental tasks in Natural Language Processing (NLP), Speech/Image recognition and many more. One of the disadvantages of Deep Learning is that in many cases, the classifier is extremely large compared to classical machine learning models. An extreme example is RoBERTa which achieves state-of-the-art results for various tasks in NLP but consists of more than $125$ millions parameters.

Large network usually requires expensive and stronger resources due to: \begin{enumerate*}[label=(\roman*)] \item slower classification time, which may be a serious limitation, especially in real-time systems such as autonomous cars, or real-time text/speech translations, \item large memory requirement, which makes it infeasible to store the network on RAM or on a device such as IoT/smartphones, and \item high energy consumption which is related to the CPU/GPU time of each classification and requires larger batteries with shorter lifespan.
\end{enumerate*}


\paragraph{Embedding matrix. }
One of the most common approaches for compressing neural networks is to treat a layer in the network as a matrix operation and then to approximate this matrix by its compressed version. This is especially relevant in a fully connected layer. Specifically, in word embedding, this layer is called \emph{embedding layer}, which is defined by the following matrix.

The input of the embedding layer consists of $d$ input neurons, and the output has $n$ neurons. The $nd$ edges between these layers define a matrix $A\in\REAL^{n\times d}$. Here, the entry $A_{i,j}$ in the $i$th row and $j$th column of $A$ is equal to the weight of the edge between the $j$th input neuron to the $i$th output neuron. Suppose that a test sample (vector) $x\in\REAL^d$ is received as an input. The corresponding output $n$-dimensional vector is thus $y=Ax$. To simply read a column from $A$ during training, a standard vector $x$ (a row of the identity matrix) is used and is called \emph{one-hot vector}.

\paragraph{$\ell_2$ $k$-Rank Approximation.} One of the natural and common matrix approximations, including in the context of network compression, is the $\ell_2$ $k$-rank approximation, see e.g.~\cite{yu2017compressing,acharya2019online} and references therein. This is the matrix which minimizes the Frobenius norm, i.e., the sum of squared distances
$\norm{A-A_k}^2_F:=\sum_{i=1}^n \norm{A^{(i)}-A_k^{(i)}}_2^2$ between the $i$th row $A^{(i)}$ in $A$ and its corresponding row $A_k^{(i)}$ in $A_k$, over every rank $k$ matrix $A_k$. It can be easily computed via the Singular Value Decomposition (SVD) in $O(\min\br{nd^2,dn^2})$ time. Although $A_k$ has the same size as $A$, due to its low rank, it can be factorized as $A_k=UW$ where $U\in \REAL^{n\times k}$ and $W\in \REAL^{k\times d}$. We can then replace the original embedding layer that corresponds to $A$ by a pair of layers that correspond to $U$ and $W$. Storing $U$ and $W$ takes memory of $O(k(n+d))$ numbers, compared to the $O(nd)$ entries in $A$. Moreover, the computation of the output $y_k:=A_kx$  takes $O(k(n+d))$ time, compared to the $O(nd)$ time that it takes to compute $Ax$.

\paragraph{Fine tuning.} The layers that corresponds to the matrices $U$ and $W$ above are sometimes used only as initial seeds for a training process that is called \emph{fine tuning}. Here, the training data is fed into the network, and unlike the $\ell_2$ error, the error is measured with respect to the final classification. Hence, the structure of the data remains the same but the edges are updated in each iteration.

\paragraph{Application to fully connected layer.} Rank approximation can be computed independently, for every fully connected layer in a network. In this case, an activation function $f:\REAL^d\to\REAL$ is applied on the output $y=Ax$ or each of its coordinates (as Relu) to obtain $f(Ax)$. The hope is that, since $A_k$ approximates $A$ in some sense, the value $f(A_kx)$ would also approximate $f(Ax)$ for the next layers. The matrix $A_k$ is replaced by two smaller layers $U,W$ as explained above.

\section{Motivation}
\label{sec:motiv}
Geometrically, each row of $A$ corresponds to a $d$-dimensional vector (point) in $\REAL^d$, and the corresponding row in $A_k$ is its projection on a $k$-dimensional subspace of $\REAL^d$. This subspace (which is the column space of $U$ above) minimizes the sum of squared distances to the rows of $A$ over every $k$-subspace in $\REAL^d$.

Statistically, if these $n$ points were generated by adding a Gaussian noise to a set of $n$ points on a $k$-dimensional subspace, then it is easy to prove that most likely (in the sense of Maximum-Likelihood) this subspace is $U$.

The disadvantage of $\ell_2$ $k$-rank approximation is that it is optimal under the above statistical assumption, which rarely seems to be the case for most applications.

In particular, minimizing the sum of squared distances is heavily sensitive to outliers~\cite{bermejo2001oriented}; see Figure~\ref{fig:l1vsl2illus}. As explained in~\cite{wiki:Mean_squared_error}, this is the result of squaring each term, which effectively weights large errors more heavily than small ones.

This undesirable property, in many applications, has led researchers to use alternatives such as the mean absolute error (MAD), that minimizes the $\ell_1$ (sum of distances) of the error vector. For example, Compressed Sensing~\cite{donoho2006compressed} uses $\ell_1$ approximation as its main tool to clean corrupted data~\cite{huang2015two} as well as to obtain sparsified embeddings with provable guarantees as explained e.g. in~\cite{donoho2003optimally}.

In machine learning, the $\ell_1$-approximation replaces, or at least is combined, with the $\ell_2$ approximation.
Examples in scikit-learn include Lasso Regression, Elastic-Nets, or MAD error in decision trees~\cite{scikit-learn}.

\paragraph{Novel approach: Subspace Approximation meets Deep Learning.}
In the context of embedding layers, it is then natural to generalize the above $\ell_2$ approximation to $\ell_1$ $k$-rank approximation, or $\ell_p$ approximation for more general $p<2$. Geometrically, we wish to compute the $k$-subspace that minimizes the sum of (non-squared) distances to the given set of $n$ points, or the distances to the power of $p$.
According to the theory of robust statistics this should result in more accurate compressed networks that are more robust to outliers and classifications mistakes.

The main challenge is that computing the $k$-subspace that minimizes the sum of each distance to the power of $p$ is a hard problem as explained in~\cite{clarkson2015input}. In the recent decade, especially in the recent years, there was a great progress in this area, with many provable approximations. However, these approximation are either impractical or based on ad-hoc heuristics with no provable bounds. This may be related to the fact that the $\ell_p$ $k$-rank approximation problem is NP-hard problem for general (including constant) values of $p$; see Section~\ref{related}.

This motivates the questions that are answered affirmably in this paper:
\textbf{Can we obtain smaller networks with higher accuracy by minimizing sum of non-squared errors, or any other power $p \neq 2$ of distances, instead of the $\ell_2$ $k$-rank approximation via SVD?}\\
\textbf{Can we efficiently compute the corresponding $k$-rank approximation matrix $A_k$ for these cases, similarly to SVD?}

\section{Our Contribution}
Our results are based on a novel combination between modern techniques in computational geometry and applied deep learning.
We expect that future papers would extend this approach as suggested in Section~\ref{sec:futurework}.

The applied results, which were the main motivations for this papers are:
\begin{enumerate}
\item New approach for compressing networks based on $k$-rank $\ell_p$-approximation instead of $\ell_2$, for $p\in [1,\infty)$. The main motivation is the robustness to outliers and noise, which is supported by many theoretical justifications.
\item Provable algorithms for computing this $\ell_p$-approximation of every $n\times d$ matrix $A$. The deterministic version takes time $O(nd^5\log{n})$ and the randomized version takes $O(nd\log{n})$.
The approximation factor depends polynomially on $d$ and independent of $n$ for the deterministic version, and only poly-logarithmic in $n$ for the randomized version.
\item Experimental results that confirm the efficiency of our approach. For example we achieve $28\%$ compression of RoBERTa's embedding layer with only $0.63\%$ additive drop in the accuracy (without fine-tuning) in average across the tasks in GLUE. Also, we succeed to further compress DistilBERT's embedding layer by $15\%$, while paying $2.28\%$ average drop in accuracy. Full open code is provided~\cite{opencode}.
\end{enumerate}

\paragraph{Main challenges. }The empirical success of the $\ell_p$-approximation for $p<2$, compared to $\ell_2$ is not surprising, given the analysis and numerous practical examples in other fields.
However, unlike the case $p=2$, which was solved more than a century ago~\cite{eckart1936approximation} via SVD and its variants, the $\ell_p$ approximation was recently proved to be NP-hard even to approximate up to a factor of $1 + 1/\mathrm{poly}(d)$ for $p\in [1,2)$~\cite{clarkson2015input}.
This is why in the recent years many algorithms were suggested in theoretical computer science for $\ell_p$ approximation, but their running time is exponential in $k$~\cite{feldman2011unified, clarkson2015input} and their efficiency in practice is not clear. Indeed, we could not find implementations of such provable algorithms.

To obtain efficient implementations with provable guarantees, we suggest a lee-way by allowing the approximation factor to be larger than $k$, instead of aiming for ($1+\eps)$-approximation (PTAS).
In practice, this worst-case bound seems to be too pessimistic and the empirical approximation error in our experiments is much smaller. This phenomenon is common in approximation algorithms, especially in deep learning, when the data-set has a lot of structure and is very different from synthetic worse-case artificial examples.

The main mathematical tool that we use is the L\"{o}wner Ellipsoid which generalizes the SVD case to general $\ell_p$ cases, inspired by many papers in the related work below.

\section{Related work}\label{related}
In many NLP tasks, it is very common to learn representations of natural language \cite{mikolov2013distributed,radford2018improving,le2014distributed,peters-etal-2018-deep,radford2019language}. Lately, it was shown that full-network pre-training followed by fine-tuning for a specific task gives better results than training a model from scratch~\cite{dai2015semi,radford2018improving,devlin-etal-2019-bert}. However, those pre-trained models are very large, and hard to be stored, evaluated, and trained in weak devices, e.g.,~\cite{devlin-etal-2019-bert}, show that for many NLP tasks, a network consists of $340$ million parameters achieved better results by making the hidden size larger, adding more hidden layers, and more attention heads. This work was followed by others e.g,~\cite{liu2019RoBERTa} and \cite{yang2019XLNet}.
Those large state-of-the-art models consist of hundreds of millions or even billions of parameters, which leads to GPU/TPU memory limitations in the training (fine-tuning) procedure, and storage and speed/performance issues in the evaluation (real-time usage) part.

In the context of training such giant models, some interesting approaches were suggested to reduce the memory requirement, e.g.,~\cite{chen2016training} and~\cite{gomez2017reversible}. However, those methods reduce the memory requirement at the cost of speed/performance. Later,~\cite{raffel2019exploring} proposed a way to train large models based on parallelization.
Also here, the model size and evaluation speed are still an obstacle.

\paragraph{Compressed Networks. }To this end, many papers were dedicated to the purpose of compressing neural networks in the field of NLP. Those paper are based on different approaches such as pruning~\cite{mccarley2019pruning,michel2019sixteen,fan2019reducing,guo2019reweighted,gordon2020compressing}, quantization~\cite{zafrirq8bert,shen2020q}, knowledge distillation~\cite{zhao2019extreme,sanhdistilbert,tang2019distilling,mukherjee2019distilling,liu2019attentive,sun-etal-2019-patient,jiao2019tinybert,sun-etal-2020-mobilebert}, weight sharing~\cite{lan2019albert}, and low-rank factorization~\cite{wang2019structured,lan2019albert}; see example table at~\cite{gordon_2019} for compressing BERT model.
There is no convention on which approach from the above should be used. However, recent works e.g.,~\cite{lan2019albert} showed that combining such approaches yields good results.
In our work, we suggest a low-rank factorization technique for compressing the embedding layer, e.g., of BERT. This is motivated by the fact that in many networks, $20\%-40\%$ of the network size is hidden in the word embedding matrix (often called the ``lookup table''). This technique can be combined with previous known works to obtain better compression. For example, a remarkable result of~\cite{sanhdistilbert} that is based on knowledge distillation reduces the size of a BERT model by $40\%$, while maintaining $97\%$ of its language understanding capabilities and being $60\%$ faster. However, this result does not use low-rank factorization to compress the embedding layer. We show how we can compress this network even more, and achieve better accuracy than the other factorization techniques.

\paragraph{Subspace approximation.}
The $k$-rank $\ell_2$ approximation can be solved easily in $\min\br{nd^2,d^2n}$ time.
A $(1+\eps)$ approximation can be computed deterministically in $nd(k/\eps)^{O(1)}$ time~\cite{cohen2015optimal} for every $\eps>0$, and a randomized version takes $O(\nnz{A}) + (n+d)\cdot (k/\eps)^{O(1)}$ time, where $\nnz{A}$ is the number of non-zeros entries in $A$~\cite{clarkson2017low,meng2013low,nelson2013osnap}. These and many of the following results are summarized in the seminal work at~\cite{clarkson2015input}.
However, for $p\neq 2$ even computing a multiplicative $(1+\eps)$-approximation is NP-hard when $k$ is part of the input~\cite{clarkson2015input}.
Nevertheless, it is an active research area, especially in the recent years, where techniques from computational geometry are frequently used.
The case $p \geq 1$, was introduced in the theory community at~\cite{shyamalkumar2007efficient}, and earlier the case $p = 1$ in the machine learning community by~\cite{ding2006r}.
At~\cite{shyamalkumar2007efficient}, a randomized algorithm for any $p \geq 1$ that runs in time
$nd 2^{(k/\eps)^{O(p)}}$ was suggested. The state of the art for $p\in [1,2)$ at~\cite{clarkson2015input} takes $O(\nnz{A}+(n+d)(k/\eps)^{O(1)}+2^{((k/\eps)^{O(1)})})$ time.

The time complexity for $p = 1$ was improved in~\cite{feldman2010coresets} to
$nd \cdot (k/\eps)^{O(1)} + (n+d) \cdot 2^{(k/\eps)^{O(1)}}$, and
later for general $p$ to $nd \cdot (k/\eps)^{O(1)} + 2^((k/\eps)^{O(p)})$~\cite{feldman2011unified}. The latter work, together with~\cite{varadarajan2012sensitivity}, also
gives a {\it strong coreset} for {Subspace Approximation}, i.e., a way of reducing the number of rows
of $A$ so as to obtain a matrix $A'$ so that the cost of fitting the rows of $A'$ to any $k$-dimensional
subspace $F$ is within a $1+\eps$ factor of the cost of fitting the rows of $A$ to $F$; for $p=2$ such coresets were known~\cite{feldman2016dimensionality,maalouf2020tight,clarkson2017low,maalouf2020faster} and can be computed exactly $(\eps=0)$~\cite{maalouf2019fast,jubran2019introduction}.

\paragraph{Efficient approximations. }
The exponential dependency on $k$ and hardness results may explain why we could not find (even inefficient) open or closed code implementations on the web.
To our knowledge, it is an open problem to compute larger factor approximations ($\eps\in O(1)$) in time polynomial in $k$, even in theory. The goal of this paper is to provide such provable approximation in time that is near-linear in $n$ with practical implementation, and to demonstrate our usefulness in compressed networks.

\begin{figure}[t]
    \includegraphics[width=0.49\textwidth]{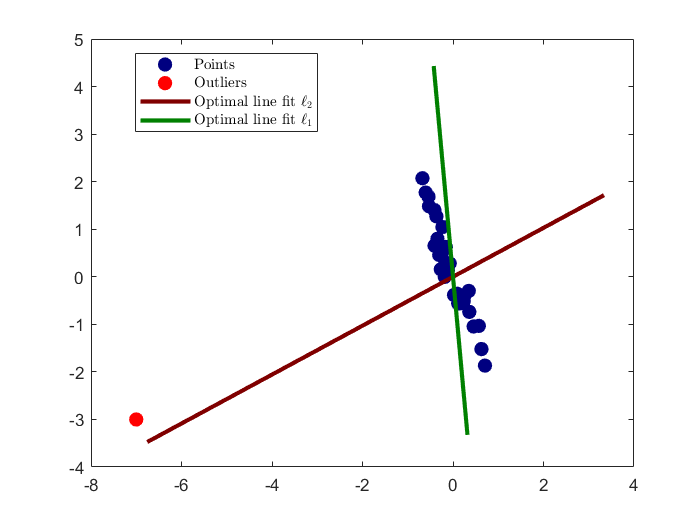}
    \caption{$\ell_1$-low rank approximation versus $\ell_2$-low rank approximation: Since the norm of a vector increases as the base of the norm decreases, the optimization problem becomes less susceptible towards outliers in the data as presented above.}
    \label{fig:l1vsl2illus}
\end{figure}

\section{Preliminaries}
\paragraph{Notations.} For a pair of integers $n,d\geq 1$, we denote by $\REAL^{n\times d}$ the set of all $n \times d$ real matrices, by $I_d \in \br{0,1}^{d \times d}$ the identity matrix, and $[n] = \br{1, \cdots, n}$.
For a vector $x \in \REAL^d$, a matrix $A \in \REAL^{n \times d}$, and a real number $p > 0$, the $p$th norm of $x$ is defined as $\norm{x}_p = \term{\sum_{i=1}^d \abs{x_i}^p}^{1/p}$, and the $\ell_p$ entry-wise norm of $A$ is defined as
$\norm{A}_{p,p} = \term{\sum_{i=1}^d \norm{Ae_i}_p^p}^{1/p}$,
where $e_i \in \br{0,1}^d$ is a vector whose $i$th entry is $1$ and $0$'s elsewhere.
We say that the columns of a matrix $A \in \REAL^{n \times d}$ (where $n \geq d$) are orthogonal if $A^TA = I_{d}$.
In addition, a matrix $F \in \REAL^{d \times d}$ is called positive definite matrix if  $F$ is a symmetric matrix, and for every $x \in \REAL^d$ such that $\norm{x}_2 > 0$, we have $x^T F x > 0$.
Furthermore, we say that a set $L \subseteq \REAL^d$ is centrally symmetric if for every $x \in L$, it holds that $-x \in L$.
Finally, a set $L \subseteq \REAL^d$ is called a convex set if for every $x, y \in L$ and $\theta \in [0,1]$,  $\theta x + (1-\theta) y \in L$.

\newcommand{\f}{\norm{\cdot}_p}
\subsection{$\f$-SVD Factorization and the L\"{o}wner ellipsoid}
In this section, we describe intuitively and formally the tools that will be used throughout the paper. Definition~\ref{def:fsvd} is based on~\cite[Definition 4]{tukan2020coresets}, while the latter defines a generic factorization for a wide family of functions, Definition~\ref{def:fsvd} focuses on our case, i.e., the function we wish to factorize is $\norm{Ax}_p$ for any $p \geq 1$, where $A\in \REAL^{n\times d}$ is the input matrix and $x$ is any vector in $\REAL^d$.

\begin{definition}[Variant of Definition 4~\cite{tukan2020coresets}]
\label{def:fsvd}
Let $A \in \REAL^{n \times d}$ be a matrix of rank $d$, and let $p \geq 1$ be a real number. Suppose that there is a diagonal matrix $D \in (0, \infty)^{d \times d}$ of rank $d$, and an orthogonal matrix $V \in \REAL^{d \times d}$, such that for every $x \in \REAL^d$,
$$
\norm{DV^Tx}_2^p \leq \norm{Ax}_p^p \leq d^{\frac{p}{2}} \norm{DV^Tx}_2^p.
$$
Define $U = A \term{D V^T}^{-1}$. Then, $UDV^T=A$ is called the $\f$-SVD of $A$.

\end{definition}

\paragraph{Why $\f$-SVD?}
The idea behind using the $\f$-SVD factorization of an input matrix $A$, is that we obtain a way to approximate the span of the column space of $A \in \REAL^{n \times d}$. This allows us to approximate the dot product $Ax$ for any $x \in \REAL^d$, which implies an approximation for the optimal solution of the $\ell_p$ low rank approximation problem.

For example, in the case of $p = 2$, the $\norm{\cdot}_2$-SVD of a matrix $A\in \REAL^{n \times d}$ is equivalent to the known SVD factorization $A=UDV^T$. This holds due to the fact that the columns of the matrix $U$ are orthogonal, and for every $x\in \REAL^d$ we have
$$\norm{Ax}_2^2 = \norm{UDV^Tx}_2^2 =\norm{DV^Tx}_2^2.$$

As for the general case of any $p \geq 1$,~\cite{tukan2020coresets} showed that the $\f$-SVD factorization always exists, and can be obtained using the \emph{L\"{o}wner} ellipsoid.

\begin{theorem}[Variant of Theorem III~\cite{john2014extremum}]
\label{thm:lowner}
Let $D \in [0, \infty)^{d \times d}$ be a diagonal matrix of full rank and an orthogonal matrix $V \in \REAL^{d \times d}$, and let $E$ be an ellipsoid defined as
\[
E = \br{x \middle| x \in \REAL^d, x^TVD^TDV^Tx \leq 1}.
\]

Let $L$ be a centrally symmetric compact convex set. Then there exists a unique ellipsoid $E$ called the \emph{L\"{o}wner ellipsoid} of $L$ such that
\[
1/\sqrt{d} E \subseteq L \subseteq E,
\]
where $1/\sqrt{d}E = \br{1/\sqrt{d} x \middle| x \in E}$.
\end{theorem}

\paragraph{Computing $\f$-SVD via L\"{o}wner ellipsoid.}
Intuitively speaking, for an input matrix $A\in \REAL^{n\times d}$, the $\f$-SVD $A=UDV^T$ aims to bound from above and below the cost function $\norm{Ax}^p_p$ for any $x\in \REAL^d$ by the term $\norm{DV^Tx}^p_2$.
Since $\norm{Ax}_p^p$ is a convex continuous function (for every $x \in \REAL^d$), the level set $L = \br{x \middle| x \in\REAL^d, \norm{Ax}_p \leq 1}$ is also convex. Having a convex set enables us to use the \emph{L\"{o}wner ellipsoid} which in short is the minimum volume enclosing ellipsoid of $L$. In addition, contracting the \emph{L\"{o}wner ellipsoid} by $\sqrt{d}$, yields an inscribed ellipsoid in $L$.
It turns out that $D,V$ of the $\f$-SVD represent the L\"{o}wner ellipsoid of $L$ as follows:
$D$ is a diagonal matrix such that its diagonal entries contains the reciprocal values of the ellipsoid axis lengths, and $V$ is an orthogonal matrix which is the basis of the same ellipsoid.
Using the containing and inscribed ellipsoids (the L\"{o}wner ellipsoid and its contracted form), enables us to bound $\f$ using the mahalonobis distance.
Although, in traditional $k$ $\ell_2$-low rank factorization with respect to an input matrix $A \in \REAL^{n \times d}$, the optimal result is equal to the sum of the smallest $d-k$ singular values, we generalize this concept to $\ell_p$-low rank factorization. Specifically, the singular values of $D$ (the reciprocal values of the ellipsoid axis lengths) serve as a bound on the \say{$\ell_p$ singular values of $A$}.

\section{Additive Approximation for the $\ell_{p}$-Low Rank Factorization}
In what follows, we show how to compute an approximated solution for the $\ell_{p}$-low rank factorization, for any $p \geq 1$; see Algorithm~\ref{algo:l1lowrank}. This is based on the $\f$-SVD factorization (see Definition~\ref{def:fsvd}).

\paragraph{From $\f$-SVD to $\ell_{p}$-Low Rank Factorization.} For any $k \in [d-1]$ and any matrix $A \in \REAL^{n \times d}$, the $\ell_{p}$-low rank factorization problem aims to minimize $\norm{A - AXX^T}_{p,p}^p$ over every matrix $X \in \REAL^{d \times k}$ whose columns are orthogonal. As a byproduct of the orthogonality of $X$, the problem above is equivalent to minimizing $\norm{AYY^T}_{p,p}^p$ over every matrix $Y \in \REAL^{d \times \term{d - k}}$ whose columns are orthogonal such that $YY^T = I_d - XX^T$. By exploiting the definition of the entry-wise $\ell_p$ norm of $AYY^T$, we can use $\f$-SVD to bound this term from above and below using the mahalonobis distance. Furthermore, we will show that using the $\f$-SVD, we can compute a matrix $A_k$ of rank $k$ such that entry-wise $\ell_p$ norm of $A - A_k$ depends on the ellipsoid axis lengths; see Algorithm~\ref{algo:l1lowrank} and Theorem~\ref{thm:additive}.


\paragraph{Overview of Algorithm~\ref{algo:l1lowrank}.}
Algorithm~\ref{algo:l1lowrank} receives as input a matrix $A \in \REAL^{n \times d}$ a positive integer $k \in [d-1]$, and a positive number $p \geq 1$, then outputs a matrix $A_k$ of rank $k$, that satisfies Theorem~\ref{thm:additive}. At Line~\ref{algl1:l1}, we compute a pair of matrices $D,V \in \REAL^{d \times d}$ such that the ellipsoid $E := \br{ x \in \REAL^d \middle| x^T VD^TDDV^T x \leq 1}$ is the L\"{o}wner ellipsoid of $L := \br{x \in \REAL^d \middle| \norm{Ax}_p \leq 1}$, where $D$ is a diagonal matrix of rank $d$ and $V$ is an orthogonal matrix. At Line~\ref{algl1:l2}, we compute the matrix $U$ from the $\f$-SVD of $A$; see Definition~\ref{def:fsvd}.  At Lines~\ref{algl1:l3}--~\ref{algl1:l4}, we set $D_k$ to be the diagonal matrix of $d \times d$ entries where the first $k$ diagonal entries are identical to the first $k$ diagonal entries of $D$, while the rest of the matrix is set to $0$.

\begin{algorithm}[htb!]
\SetAlgoLined
\DontPrintSemicolon 
\KwIn{A matrix $A \in \REAL^{n \times d}$ of rank $d$, $p \geq 1$, a positive integer $k \in [d-1]$, and a positive real number $p \geq 1$.}
\KwOut{A matrix $U \in \REAL^{n \times d}$, a diagonal matrix $D_k \in [0, \infty)^{d \times d}$, an orthogonal matrix $V \in \REAL^{d \times d}$ where $U,V$ are from the $\f$-SVD of $A$, and a set of $d$ positive real numbers $\br{\sigma_1, \ldots, \sigma_d}$.}
\vspace{1.5mm}
$\term{D,V} := \Lowner{A, p}$ \label{algl1:l1} \tcp{See Algorithm~3 at the Appendix}
$U := A\term{DV^T}^{-1}$ \label{algl1:l2} \tcp{computing $U$ from the $\f$-SVD of $A$ with respect to the $\ell_p$-regression problem}
$\br{\sigma_1, \ldots, \sigma_d} := $ the diagonal entries of $D$ \label{algl1:l3}\;
$D_k := \mathrm{diag}\term{\sigma_1, \ldots, \sigma_{k}, 0, \ldots, 0}$ \label{algl1:l4}\tcp{A diagonal matrix in $\REAL^{d \times d}$}
\Return{$U, D_k, V, \br{\sigma_1, \ldots, \sigma_d}$} \label{algl1:l5}\;
\caption{\lowRankOne{A,k, p}}
\label{algo:l1lowrank}
\end{algorithm}

\paragraph{Deterministic result.} In what follows, we present our deterministic solution for the $\ell_p$-low rank factorization problem.
\begin{restatable}{theorem}{additiveapprox}
\label{thm:additive}
Let $A \in \REAL^{n \times d}$ be real matrix, $p \geq 1$, $k \in [d-1]$ be an integer and let $\term{U,D_k,V,\br{\sigma_{1}, \ldots \sigma_{d}}}$ be the output of a call to $\lowRankOne{A,k,p}$.
Let $A_k = U D_k V^T$. Then
\[
d\sigma_{d}^p \leq \norm{A - A_k}_{p,p}^p \leq d^{1+\frac{p}{2}} \sigma_k^p.
\]
\end{restatable}

\paragraph{Note} that the set $\br{\sigma_i}_{i=1}^d$ denote the reciprocal values of the ellipsoid $E$ axis's lengths where $E$ is the L\"{o}wner ellipsoid of $L = \br{x \in \REAL^d \middle| \norm{Ax}_p \leq 1}$. As discussed in the previous section, these values serve to bound the \say{$\ell_p$ singular values or $A$}.

\paragraph{Randomized result.} In addition to our deterministic result, we also show how to support a randomized version that computes an approximation in a faster time. The intuition, proof, and explanation for this result is given in the technical appendix.
\begin{restatable}{theorem}{additiveapproxcor}
\label{coro:additive}
Let $A \in \REAL^{n \times d}$ be real matrix, $p \geq 1$, $k \in [d-1]$ be an integer.
There exists a randomized algorithm which when given a matrix $A \in \REAL^{n \times d}$, $k \in [d-1]$, in time $O\term{n d \log{n}}$ returns $\term{U,D_k,V,\br{\sigma_{1}, \ldots \sigma_{d}}}$, such that
\[
d \sigma_{d}^p \leq \norm{A - A_k}_{p,p}^p \leq d^{1 + p} \term{d^3 + d^2 \log{n}}^{\abs{1 - p/2}} \sigma_{k+1}^p,
\]
holds with probability at least $1 - \frac{1}{n}$, where $A_k = U D_k V^T$.
\end{restatable}

\begin{remark}
Note that in our context of embedding layer compression, the corresponding embedding matrix $A$ has more columns than rows. Regardless, our $\ell_{p}$ norm of any $A - B$ such that $A,B \in {d \times n}$ enables us to have
\[
\norm{A-B}_{p,p}^p = \norm{A^T-B^T}_{p,p}^p.
\]
Hence, by substituting $A := A^T$ and $A_k^T := A_k^T$ yields
\[
d \sigma_{d}^p \leq \norm{A - A_k}_{p,p}^p \leq d^{1 + \frac{p}{2}} \sigma_{k+1}^p.
\]
for our deterministic results, and similarly we can obtain our randomized result.
\end{remark}

\section{Experimental Results}
\label{sec:results}

\begin{table*}[t!]\label{table:bestmodels}
\begin{center}
\begin{tabular}{cc|P{2.5cm}|c c c c c c c | c}
 \hline
 \multicolumn{2}{c|}{Model} & Embedding layer compression rate & MRPC & COLA & MNLI & SST-2 & STS-B & QNLI & RTE & Avg.\\
 \hline
 \multirow{3}{*}{RE-RoBERTa} & base & $15\%$ & $0.49$ & -$2.16$ & $0.01$ & $0.045$ & $0.013$ & $0.018$ & $1.08$ & $-0.072$ \\
 & small & $28\%$ & $0.98$ & $-2.01$ & $0.08$ & $0.68$ & $0.87$ & $1.33$ & $2.52$ & $0.63$ \\
 & tiny & $41\%$ & $2.69$ & $3.82$ & $2.18$ & $2.17$ & $3.10$ & $3.58$ & $2.16$ & $2.81$
 \\
 \hline

 \multirow{3}{*}{RE-XLNet} & base & $15\%$ & $2.20$ & $-0.43$ & $-0.07$ & $0.22$ & $0.03$ & $2.39$ & $2.16$ & $0.92$\\
 & small & $21\%$ & $1.47$ & $0.26$ & $0.11$ & $-0.34$ & $0.03$ & $3.42$ & $4.33$ & $1.32$\\
 & tiny & $28\%$ & $1.96$ & $3.19$ & $0.47$ & $-0.22$ & $0.19$ & $4.46$ & $6.13$ & $2.31$\\
 \hline
 \multirow{3}{*}{\centering RE-BERT} & base & $15\%$ & $0.73$ & $-0.54$ & $0.48$ & $-1.49$ & $0.85$ & $2.36$ & $1.80$ & $0.59$\\
 & small & $21\%$ & $3.43$ & $0.08$ & $1.72$ & $0.45$ & $1.62$ & $3.78$ & $1.44$ & $1.78$ \\
 & tiny & $28\%$ & $4.90$ & $-0.94$ & $3.48$ & $1.49$ & $2.66$ & $7.65$ & $1.80$ & $3$
 \\
 \hline
 RE-DistilBERT & base & $15\%$ & $1.47$ & $5.24$ & $0.86$ & $0.34$ & $0.13$ & $5.80$ & $2.16$ & $2.28$ \\
\end{tabular}
\end{center}
\caption{In the above, we present our compressed networks and their drop in accuracy based on the compression rate of the embedding layer. Specifically, each non percentile value represents the accuracy drop achieved by our compressed model with respect to its original model (e.g., RE-RoBERTa is a compressed model of RoBERTa), while negative values presents improvements in the accuracy upon the non-compressed version of the corresponding model. The last column is the average accuracy drop over all tested tasks. The ``RE'' here stands for ``Robust Embedding''.}
\label{table:bestmodels}
\end{table*}

\paragraph{The compressed networks. }
We compress several frequently used NLP networks:
\begin{enumerate*}[label=(\roman*)]
    \item BERT~\cite{devlin-etal-2019-bert},\label{bert}
    \item DistilBERT~\cite{sanhdistilbert},
    \item XLNet~\cite{yang2019XLNet},
    \item and RoBERTa~\cite{liu2019RoBERTa};\label{RoBERTa}
\end{enumerate*}
See full details on the sizes of each network and there embedding layer before compression at Table~\ref{table:bestmodels}.

\paragraph{Implementation, Software, and Hardware. }All the experiments were conducted on a AWS p$2$.xlargs machine with $1$ GPU NVIDIA K$80$, $4$ vCPUs and $61$ RAM [GiB].
We implemented our suggested compression algorithm (Algorithm~\ref{algo:l1lowrank}) in python~$3.8$ using Numpy library~\cite{van2011numpy}.
To build and train networks~\ref{bert}--\ref{RoBERTa}, we used the suggested implementation at the Transformers~\footnote{https://github.com/huggingface/transformers} library from HuggingFace~\cite{wolf2019huggingface} (Transformers version $2.3$, and PyTorch version $1.5.1$~\cite{paszke2017automatic}). Before the compression, all the networks were fine-tuned on all the tasks from the GLUE benchmark to obtain almost the same accuracy results as reported in the original papers. Since we didn't succeed to obtain close accuracy on the tasks QQP and WNLI (with most of the network), we didn't include results on them.

\paragraph{Our compression. }We compress each embedding layer (matrix) of the reported networks by factorizing it into two smaller layers (matrices) as follows.
For an embedding layer that is defined by a matrix $A\in \REAL^{n\times d}$, we compute the matrices $U,D_k,V$ by a call to $\lowRankOne{A, k, 1}$ (see Algorithm~\ref{algo:l1lowrank}), where $k$ is the low rank projection we wish to have. Observe, that the matrix $D_k$ is a diagonal matrix and its last $d-k$ columns are zero columns. We then compute a non-square diagonal matrix $D_k^\prime\in \REAL^{d\times k}$ that is the result of removing all the zero columns of $D_k$. Now, the $\ell_1$ $k$-rank approximation of $A$ can be factorized as $A_k = \term{U\sqrt{D_k^\prime}} \term{{\sqrt{D_k^\prime}}^T V^T}$. Hence, we save the two matrices (layers): \begin{enumerate*}[label=(\roman*)]
    \item $U\sqrt{D_k^\prime}$ of size $n\times k$,
    \item and ${\sqrt{D_k^\prime}}^T V^T$ of size $k\times d$.
\end{enumerate*}
This yields two layer of total size of $nk + kd$ instead of single embedding layer of total size of $nd$.

\paragraph{Reported results. }We report the test accuracy drop (relative error) on all the tasks from GLUE benchmark~\cite{wang2018glue} after compression for several compression rates.
\begin{enumerate}[label=(\roman*)]
    \item In Figure~\ref{fig:results} the $x$-axis is the compression rate of the embedding layer, and the $y$-axis is the accuracy drop (relative error) with respect to the original accuracy of the network.
    Each figure reports the results for a specific task from the GLUE benchmark on all the networks we compress. 
    Here, all reported results are compared to the known $\ell_2$-factorization using SVD. Also, in all the experiments we do not fine-tune the model after compressing, this is to show the robustness and efficiency of our technique.
    \item Table~\ref{table:bestmodels} suggests the best compressed networks in terms of accuracy VS size. For every network from~\ref{bert}--\ref{RoBERTa}, we suggest a compressed version of it with a very little drop in the accuracy, and sometimes with an improved accuracy. Given a network ``X'', we call our compressed version of ``X'' by ``RE-X'', e.g., RE-BERT and RE-XLNet. The ``RE'' here stands for ``Robust Embedding''.
\end{enumerate}

\section{Discussion}

\begin{figure*}[htb!]
    \centering
    \begin{subfigure}{.33\textwidth}
    \hspace{0.5in}
    \centering
    \includegraphics[width=1\textwidth]{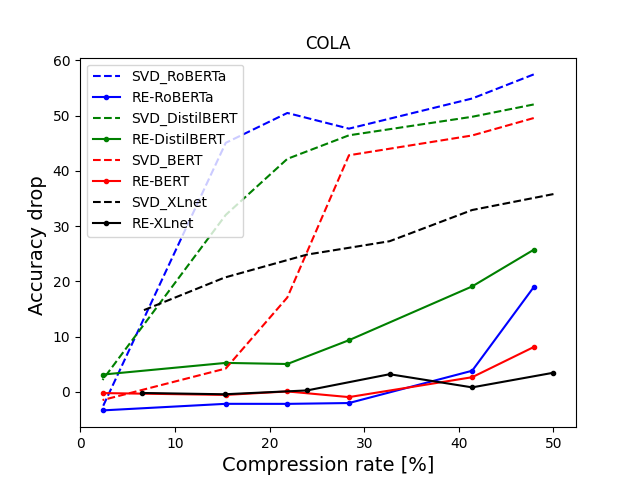}
    \end{subfigure}
    \begin{subfigure}{.33\textwidth}
    \centering
    \hspace{-0.5in}
    \includegraphics[width=1\textwidth]{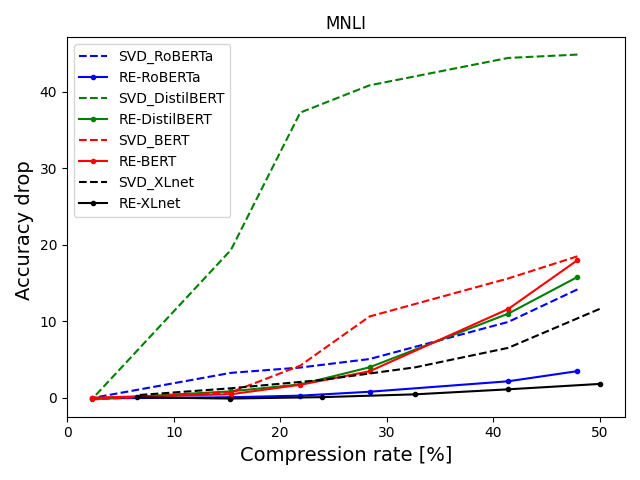}
    \end{subfigure}
    \begin{subfigure}{.33\textwidth}
    \centering
    \hspace{-0.5in}
    \includegraphics[width=1\textwidth]{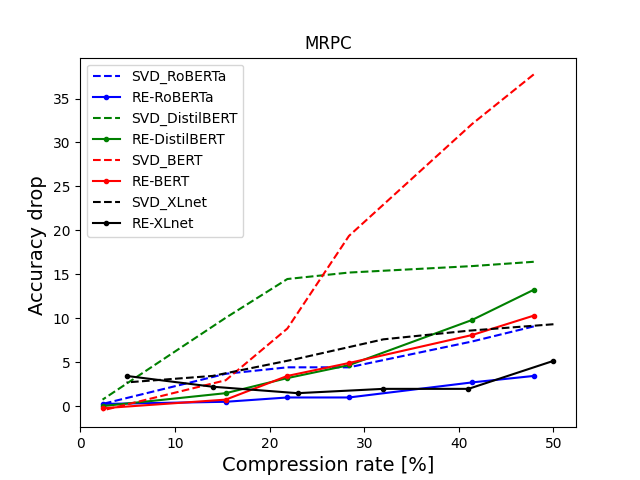}
    \end{subfigure}
    \begin{subfigure}{\textwidth}
    \hspace*{-0.14in}
    \includegraphics[width=0.25\textwidth]{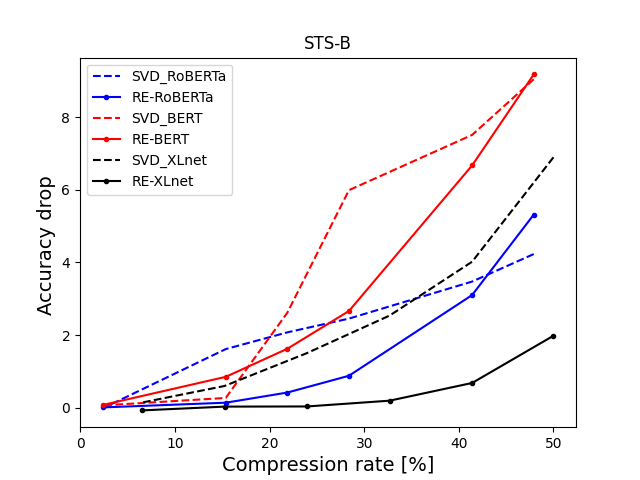}
    \includegraphics[width=0.25\textwidth]{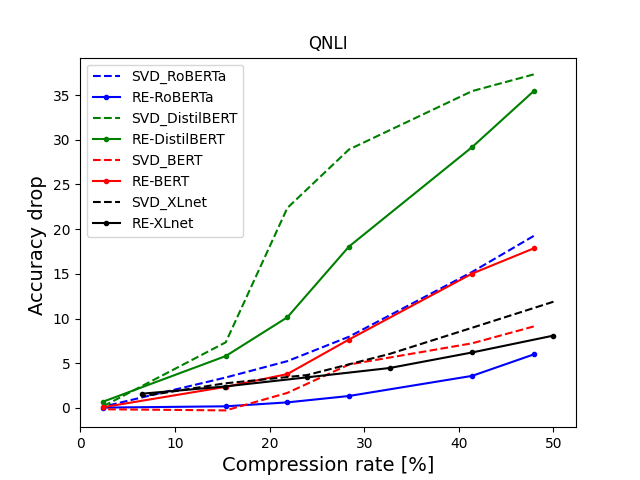}
    \includegraphics[width=0.25\textwidth]{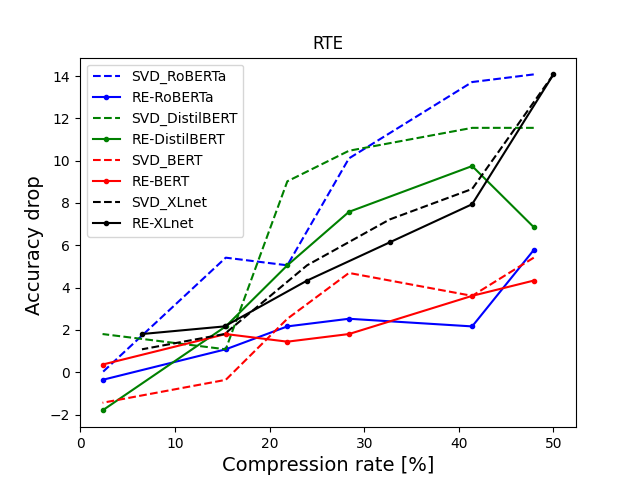}
    \includegraphics[width=0.24\textwidth]{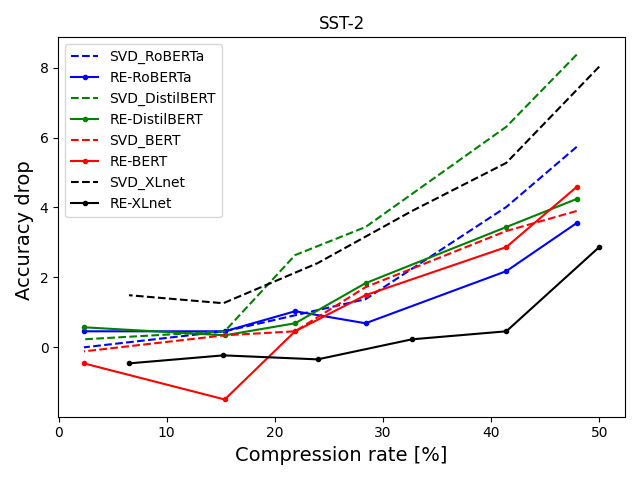}
    \end{subfigure}
    \caption{Here, we report the accuracy drop (additive error) as a function of the embedding layer's compression rate on the networks~\ref{bert}--\ref{RoBERTa}. We compare our results with SVD over several task from the GLUE benchmark. For a netowrk ``X''  Our compressed version of it is called ``RE-X'', e.g., RE-BERT and RE-XLNet.}
    \label{fig:results}
\end{figure*}

It can be seen by Figure~\ref{fig:results} that our approach is more robust than the traditional SVD. In most of the experiments our suggested compression achieves better accuracy for the same compression rate compared to the traditional SVD. Mainly, we observed that our compression schemes shines when either vocabulary is rich (the number of subword units is large) or the model itself is small (excluding the embedding layer).
Specifically speaking, in RoBERTa our method achieve better results, this is due to the fact that RoBERTa's vocabulary is rich (i.e., $50$K subword units compared to the $30$K in BERT). This large vocabulary increases the probability of having outliers in it, which is the main justification for our approach.
In DistilBERT, the network is highly efficient. This can lead to a \say{sensitive snowball effect}, i.e., the classification is highly affected by even the smallest errors caused by the compression of the embedding layer. Since SVD is sensitive to outliers and due to the fact that the network is highly sensitive to small errors, the existence of outliers highly affects the results. This phenomenon is illustrated throughout Figure~\ref{fig:results}. Here, our compression scheme outperforms the SVD due to its robustness against outliers which in turn, achieves smaller errors.
As for XLNet, the model encodes the \emph{relative positional embedding}, which in short, represents an embedding of the relative positional distance between words. In our context, this means that having outliers, highly affects the relative positional embedding which in turn affects the classification accuracy. Hence, this explains why we outperform SVD.
Since none of the above phenomena holds for BERT, this may explain why sometimes SVD achieves better results. However, across most tasks, our compression scheme is favorable upon SVD.

Finally, for some tasks at low compression rates, the accuracy has been improved (e.g., see task SST-2 at Figure~\ref{fig:results} when compressing BERT). This may be due to the fact that at low compression rates, we remove the least necessary (redundant) dimensions. Thus, if these dimensions are actually unnecessary, by removing them, we obtain a generalized model which is capable to classify better.

\begin{table}[htb!]
\begin{center}
\begin{tabular}{cc|c|c}
 \hline
 \multicolumn{2}{c|}{Model} & Embedding layer size & Parameters \\ \hline
 BERT & base & $30522 \times 768$ & $110M$ \\
 \hline
 RoBERTa & base & $50265 \times 768$ & $125M$ \\
 \hline
 XLNet & base & $32000 \times 768$ & $110M$ \\
 \hline
 DistilBERT & base & $30522 \times 768$ & $66M$ \\
 \hline
 \end{tabular}
\end{center}
\caption{The sizes of the networks BERT, RoBERTA, XLNet, and DistilBERT, and there embedding layers.}
\label{table:modelsizes}
\end{table}
\section{Conclusion and Future Work}
\label{sec:futurework}
We provided an algorithm that computes an approximation for $\ell_p$ $k$-rank approximation, where $p \geq 1$. We then suggested a new approach for compressing networks based on $k$-rank $\ell_p$-approximation where $p\in [1,2]$ instead of $\ell_2$. The experimental results at section~\ref{sec:results}, showed that our suggested algorithm overcomes the traditional $\ell_2$ $k$-rank approximation and achieves higher accuracy for the same compression rate.

Future work includes: \begin{enumerate*}[label=(\arabic*)]
    \item Extending our approach to other factorization models, such as Non-Negative Matrix Approximation or Dictionary Learning.
    \item Experimental results on other benchmarks and other models,
    \item suggesting algorithms for the $\ell_p$ $k$-rank approximation for any $p \in (0,1)$, while checking the practical contribution in compressing deep networks for this case, and
    \item combining this result with other compression techniques to get smaller network with higher accuracy.
\end{enumerate*}

\bibliography{main}
\appendix
\section{Our algorithms}
\begin{figure}[!b]
    \includegraphics[width=0.49\textwidth]{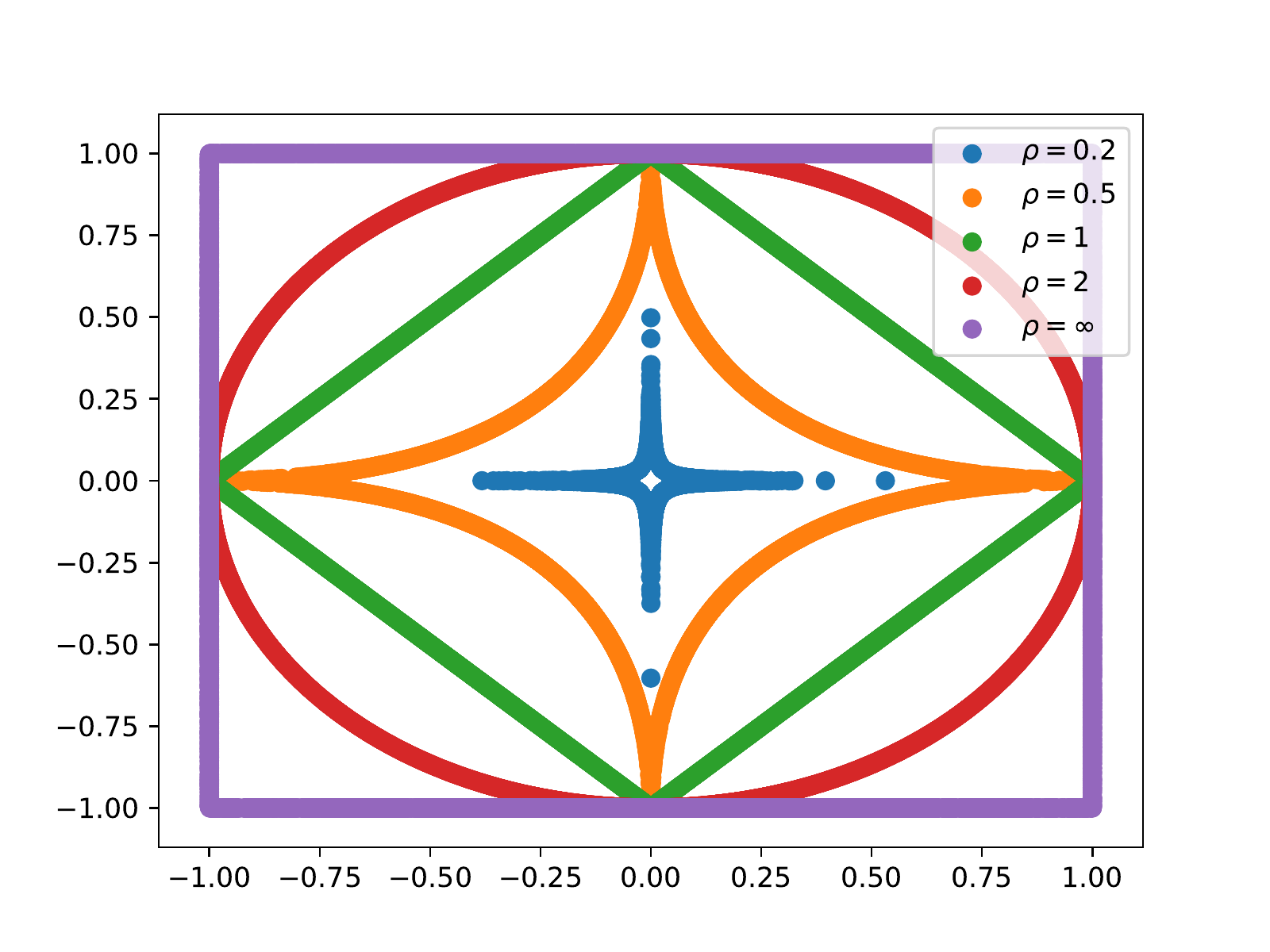}
    \caption{Unit balls in different $p$-spaces: the lower the $p$, the higher the norm of the vector.}
    \label{fig:norms}
\end{figure}
\subsection{Optimal solution for $\ell_2$ Low Rank Approximation}

We start with the simple known case of the $\ell_{2}$-low rank approximation, where the optimal solution can be computed using the known SVD factorization; see Algorithm~\ref{algo:l2lowrank}. Then, we show how we can compute an approximated solution for the general case of $\ell_{p}$-low rank approximation, for any $p \geq 1$. This is based on the $\f$-SVD factorization (see Definition~\ref{def:fsvd}) that generalizes the idea of SVD to any norm; see Algorithm~\ref{algo:l1lowrank}.

\begin{algorithm}[htb!]
\SetAlgoLined
\DontPrintSemicolon 
\KwIn{A matrix $A \in \REAL^{n \times d}$ of rank $d$, and a positive integer $k \in [d]$.}
\KwOut{A matrix $A_k \in \REAL^{n \times d}$ of rank $k$ which is an optimal solution of the $\ell_2$ $k$-rank approximation problem}
\vspace{1.5mm}
$(U,D,V) :=$ the \emph{SVD} decomposition of $A$\; \label{algl2:l1} \tcp{Computed using Singular value decomposition algorithm}
$D_k$ := a $d \times d$ diagonal matrix where the first $k$ entries of its diagonal are identical to the first $k$ entries of the diagonal of $D$ and the rest are zeros.\; \label{algl2:l2}
$A_k := U D_k V$ \label{algl2:l3}\tcp{A $n \times d$ matrix of rank $k$}
\Return{$A_k$}\; \label{algl2:l4}
\caption{$\lowRankTwo{A,k}$}
\label{algo:l2lowrank}
\end{algorithm}

Given a matrix $A \in \REAL^{n \times d}$ of rank $d$, and an integer $k \in [d-1]$, Algorithm~\ref{algo:l2lowrank} computes a matrix $A_k \in \REAL^{n \times d}$ of rank $k \in [d-1]$ which is an optimal solution for $\ell_2$ $k$-rank approximation problem.
The heart of the algorithm lies in computing the SVD decomposition of $A$ at Line~\ref{algl2:l1}, where we compute an orthogonal matrix $U \in \REAL^{n \times d}$, a diagonal matrix $D \in [0,\infty)^{d \times d}$, and an orthogonal matrix $V \in \REAL^{d \times d}$ such that $A = UDV^T$. At Line~\ref{algl2:l2}, we zero out the $d-k$ entries of $D$ and we call the new matrix $D_k$, and finally at Line~\ref{algl2:l3} we set the optimal solution to the $\ell_2$ $k$-rank approximation problem involving $A$ to be $A_k$. This is guaranteed based on the Eckart, Young and Minsky Theorem~\cite{eckart1936approximation}.

\subsection{Computing the L\"{o}wner ellipsoid}
For an input matrix $A \in \REAL^{n \times d}$ of rank $d$, and a number $p \geq 1$, we now show how to compute the \emph{L\"{o}wner ellipsoid} for the set  $L := \br{x \middle| x \in \REAL^d, \norm{Ax}_p \leq 1}$. This is a crucial step towards computing the $\f$-SVD (see Definition~\ref{def:fsvd}) for the matrix $A$ in the context of the $\ell_{p}$-low rank approximation problem, which will allow us to suggest an approximated solution; See Theorem~\ref{thm:additive}.

\paragraph{Overview of Algorithm~\ref{alg:Lowner} (computing the L\"{o}wner ellipsoid).} Algorithm~\ref{alg:Lowner} receives as input a matrix $A \in \REAL^{n \times d}$ of rank $d$, and a number $p \geq 1$. It outputs a L\"{o}wner ellipsoid for the set $L$ (see Line~\ref{algLown:l1} of Algorithm~\ref{alg:Lowner}). First at Line~\ref{algLown:l1} we initialize $L$ to be set of all the points $x$ in $\REAL^d$ such that $\norm{Ax}_p \leq 1$. At Lines~\ref{algLown:l2}--\ref{algLown:l4} we find a ball $E$ in $\REAL^d$ of radius $r$, which contains the set $L$ and its center is set to be the origin $0_d$. Then, we build a diagonal matrix $F$ where we set its diagonal entries to $r$.

Lines~\ref{algLown:l7}--\ref{algLown:l11} represent the pseudo-code of the basic Ellipsoid method which is described in details at~\cite{grotschel1993ellipsoid}, where we set $H$ to the separating hyperplane between $c$ (the center of the ellipsoid $E$) and $L$, $b$ is set to be the multiplication between $F$ and the normalized subgradient of $\norm{Ax}_p$ at $x = c$, where $b$ is used to set the next candidate ellipsoid; See Figure~\ref{fig:lownerVisualization} for visualization of this method.

In Lines~\ref{algLown:l12}--\ref{algLown:l16}, we compute the next candidate ellipsoid $E$, and based on it we set $\Tilde{V}$ to be the set containing the vertices of the inscribed ellipsoid $\frac{1}{d}\term{E - c} + c$ in $L$. Now, if $\Tilde{V} \subseteq L$ then we halt the algorithm, otherwise we find the farthest vertex point $v$ in $\Tilde{V}$ from $L$ with respect to $\norm{Ax}_p$, and finally we set $H$ to be the separating hyperplane between $v$ and $L$.

Lines~\ref{algLown:l18}--\ref{algLown:l24} presents the pseudo code of applying a shallow cut update to the ellipsoid $E$, this is described in details at~\cite{grotschel1993ellipsoid}.
Finally, at Line~\ref{algLown:l26}, we set $G$ to be the Cholesky decomposition of $F^{-1}$; See \cite{golub2012matrix} for more details. For formal details see Theorem~\ref{thm:lowner}

\begin{figure*}[htb!]
    \centering
    \begin{subfigure}{.3\textwidth}
    \centering
    \includegraphics[width=1\textwidth]{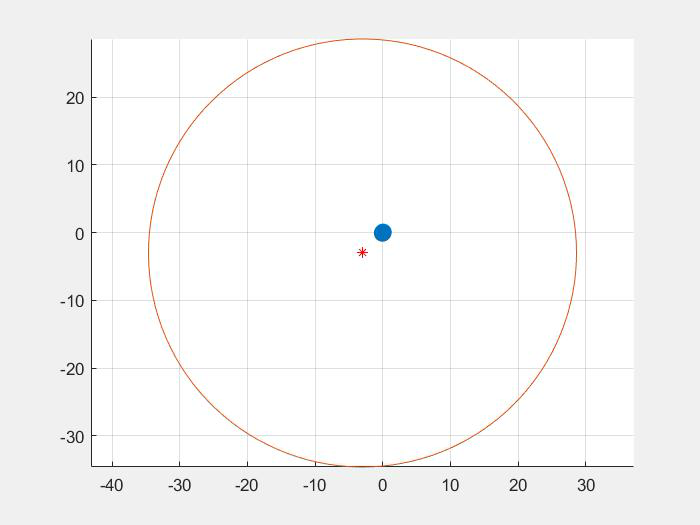}
    \caption{}
    \end{subfigure}
    \begin{subfigure}{.3\textwidth}
    \centering
    \includegraphics[width=1\textwidth]{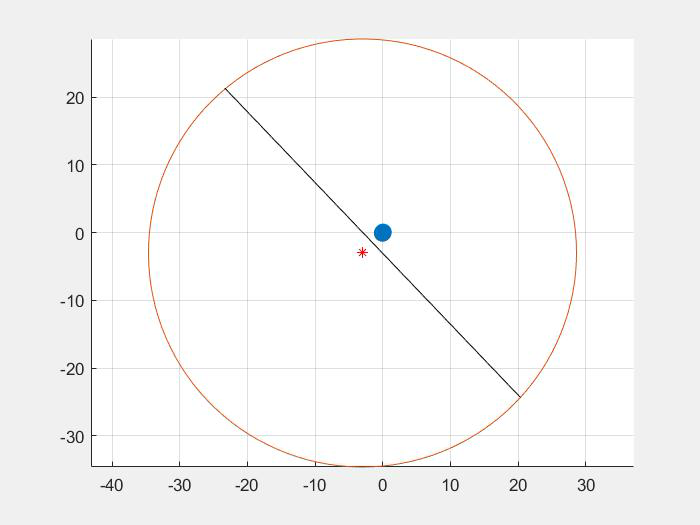}
    \caption{}
    \label{fig:step2}
    \end{subfigure}
    \begin{subfigure}{.3\textwidth}
    \centering
    \includegraphics[width=1\textwidth]{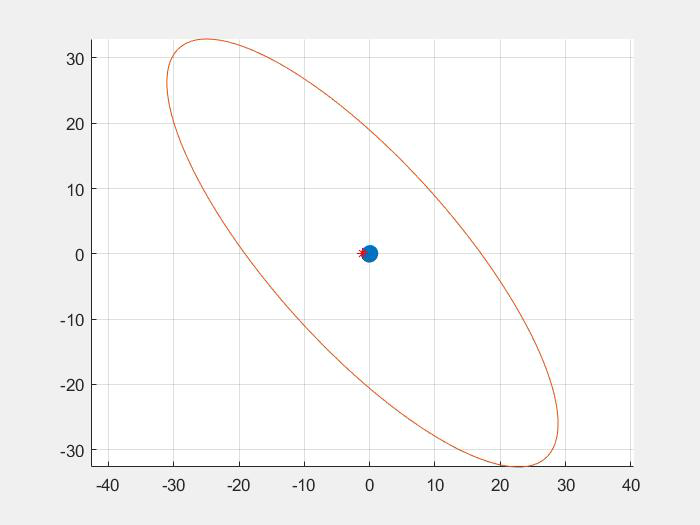}
    \caption{}
    \label{fig:step3}
    \end{subfigure}
    \begin{subfigure}{.3\textwidth}
    \centering
    \includegraphics[width=1\textwidth]{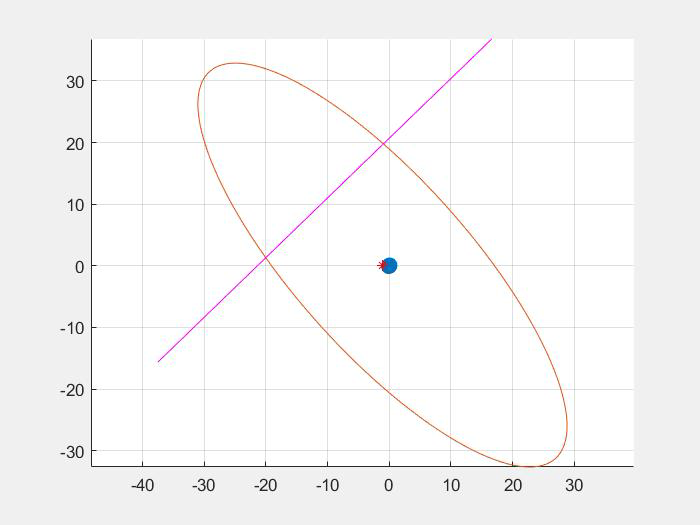}
    \caption{}
    \label{fig:step4}
    \end{subfigure}
    \begin{subfigure}{.3\textwidth}
    \centering
    \includegraphics[width=1\textwidth]{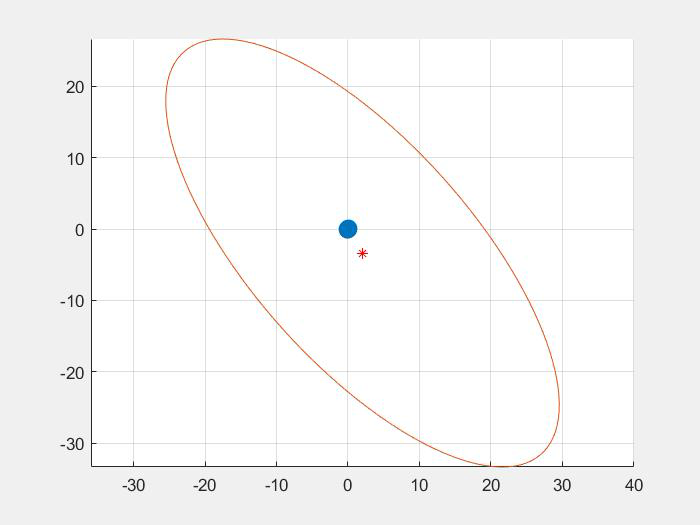}
    \caption{}
    \label{fig:step5}
    \end{subfigure}
    \begin{subfigure}{.3\textwidth}
    \centering
    \includegraphics[width=1\textwidth]{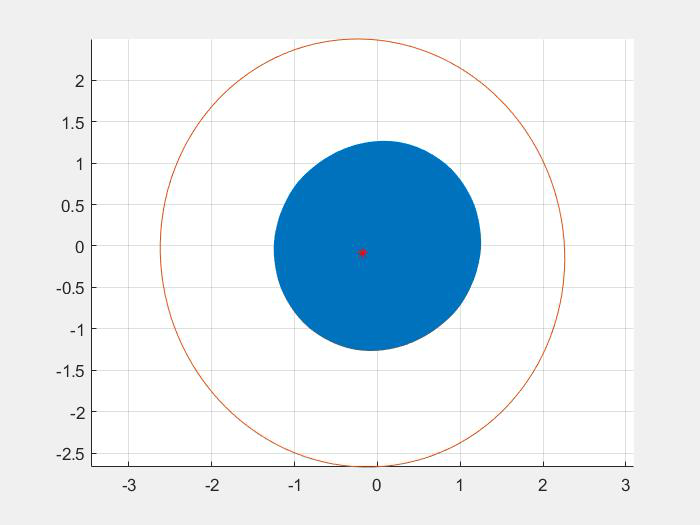}
    \caption{}
    \label{fig:step6}
    \end{subfigure}
    \caption{\textit{Computing the L\"{o}wener ellipsoid:} \textit{Step~\RNum{1}}: We start with an ellipsoid that contains our level-set (the blue body). From here, the basic ellipsoid method is invoked, i.e. while the center is not contained inside the level set (blue body), a separating hyperplane between the center of the ellipsoid and the level set is computed, and the ellipsoid is stretched in a way such that the center is getting closer in distance to the level set. The basic ellipsoid method halts when the center is contained in the level set; see Figure~\ref{fig:step2}--\ref{fig:step3} for illustration of the ellipsoid method. \textit{Step \RNum{3}:} We compute a contracted version of the current ellipsoid, and check if all of its vertices are contained in the level set.
    If there exists one ellipsoid's vertex which is not contained in the level set, we find the farthest vertex of the contracted ellipsoid from the level set and compute a separating hyperplane between it and the level set, which then the ellipsoid is stretched such that this vertex becomes closer to the level set presented in Figures~\ref{fig:step4}--\ref{fig:step5}.
    When then loop over Steps\RNum{2}--\RNum{3} until  the contracted ellipsoid's vertices are contained in the level set; see Figure~\ref{fig:step6}.}
    \label{fig:lownerVisualization}
\end{figure*}

\begin{algorithm*}
\SetAlgoLined
\DontPrintSemicolon 
\KwIn{A matrix $A \in \REAL^{n \times d}$ of rank $d$, and $p \geq 1$}
\KwOut{A diagonal matrix $D \in [0, \infty)^{d \times d}$ and an orthogonal matrix $V \in \REAL^{d \times d}$ satisfying Theorem~\ref{thm:lowner}.}
\vspace{1.5mm}
$L := \br{x \middle| x \in \REAL^d, \norm{Ax}_p \leq 1}$ \label{algLown:l1}\;
$E := $ A ball centered around the origin which contains $L$ \label{algLown:l2}\;
$r := $ the radius of $E$ \label{algLown:r}\;
$F := r I_d$ \label{algLown:l3}\;
$c := \begin{bmatrix} 0\\ \vdots \\ 0\end{bmatrix} \in \REAL^d$ \label{algLown:l4}\;
\While{$\mathrm{true}$ \label{algLown:l5}}{
    \While{$c \not\in L$\label{algLown:l6}}{ \tcc{Here we apply the basic ellipsoid method}
        $grad :=  \norm{Ac}_p^{p-1} \cdot \norm{Ac}_{p-1}^{p-1} \cdot \begin{bmatrix} \sum\limits_{i=1}^n \frac{A_{i,1} c_1}{\abs{A_{i*}^Tc }},
        \cdots,
        \sum\limits_{i=1}^n \frac{A_{i,d} c[d]}{\abs{A_{i*}^Tc}}
        \end{bmatrix}^T$ \label{algLown:l7}\tcp{The gradient of $\ell_p$-regression at $c$}
        $H := \frac{1}{\norm{grad}_\infty} grad$ \label{algLown:l8} \tcp{Separating hyperplane between the center of the ellipsoid and $L$}
        $b := \frac{1}{\sqrt{H^T F H}} F H$ \label{algLown:l9}\;
        $c := c  - \frac{1}{d+1}b$ \label{algLown:l10}\tcp{Updating the center of the ellipsoid}
        $F := \frac{d^2}{d^2 -1} \term{F - \frac{2}{d+1} b b^T}$ \label{algLown:l11} \tcp{updating the basis of the ellipsoid}
    }
    $E := \br{ x \in \REAL^d \middle|\term{x-c}^T F^{-1} \term{x-c} \leq 1}$ \label{algLown:l12}\;
    $\Tilde{V} := $ vertices of $\frac{1}{d}\term{E - c} + c$ \label{algLown:l13} \tcp{compute the vertices of the enclosed version of $E$ in $L$}
    \If{$\Tilde{V} \subseteq L$ \label{algLown:l14}}{
    Go to Line $27$ \label{algLown:l15}\;
    }
    $v := \argmax{x \in \Tilde{V}} \norm{Ax}_p$ \label{algLown:l16} \tcp{the farthest vertex of the contained ellipsoid from $L$ by the norm $\norm{Ax}_p$}
    $grad :=  \norm{Av}_p^{p-1} \cdot \norm{Av}_{p-1}^{p-1} \cdot \begin{bmatrix} \sum\limits_{i=1}^n \frac{A_{i,1} v_1}{\abs{A_{i*}^Tv }} ,
        \cdots,
        \sum\limits_{i=1}^n \frac{A_{i,d} v_d}{\abs{A_{i*}^Tv}}
        \end{bmatrix}^T$ \label{algLown:l17}\tcp{The gradient of $\ell_p$-regression at $v$}
    $H := \frac{1}{\norm{grad}_\infty} grad$ \label{algLown:l18} \tcp{Separating hyperplane between $v$ and $L$}
    \tcc{Here we apply the shallow cut ellipsoid update}
    $z := \frac{1}{\term{d+1}^2}$ \label{algLown:l19}\;
    $\sigma := \frac{d^3 \term{d+2}}{\term{d+1}^3\term{d-1}}$ \label{algLown:l20}\;
    $\zeta := 1 + \frac{1}{2 d^2\term{d+1}^2}$\label{algLown:l21}\;
    $\tau := \frac{2}{d\term{d+1}}$\label{algLown:l22}\;
    $b := \frac{1}{\sqrt{H^T F H}}F H$ \label{algLown:l23}\;
    $F := \zeta \sigma \term{F - \tau \term{b b^T}}$ \label{algLown:l24}\;
    $c := c - z b$ \label{algLown:l25}\;
}
$G := \mathrm{chol}\term{F^{-1}}$ \label{algLown:l26} \tcp{The Cholesky decomposition of $F^{-1}$}
$\term{U,D,V} := $ the \emph{SVD} of $G$\;
\Return{$D, V$}
\caption{\Lowner{A, p}}
\label{alg:Lowner}
\end{algorithm*}

\section{Proof of our main theorems}

\subsection{Proof of Theorem~\ref{thm:additive}}

\begin{claim}
\label{clm:lowest_singular_value}
Let $D \in [0,\infty)^{d \times d}$ be diagonal matrix of rank $d$, and let $\sigma > 0$ be the lowest singular value of $D$. Then for every unit vector $x \in \REAL^d$,
\[
\norm{Dx}_2 \geq \sigma.
\]
\end{claim}

\begin{proof}
Let $x \in \REAL^d$ be a unit vector, and for every $i \in [d]$, let $D_{i,i}$ denote the $i$th diagonal entry of $D$, and $x_i$ denotes the $i$th entry of $x$.
Observe that
\begin{equation*}
\begin{split}
\norm{Dx}_2 &= \term{\sum_{i=1}^d \abs{D_{i,i}x_i}^2}^{\frac{1}{2}} \geq \term{\sum_{i=1}^d \abs{\sigma x_i}^2}^{\frac{1}{2}} \\
&= \sigma \norm{x}_2 = \sigma,
\end{split}
\end{equation*}
where the first equality follows from the definition of norm, the inequality holds by definition of $\sigma$, and the last equality holds since $x$ is a unit vector.
\end{proof}

\begin{restatable}[Special case of Lemma 15~\cite{tukan2020coresets}]{lemma}{rhosvd}
\label{lem:rho_svd}
Let $A \in \REAL^{n \times d}$ be a matrix of full rank, $p \geq 1$. Then there exist a diagonal matrix $D \in [0, \infty)^{d \times d}$ of full rank and an orthogonal matrix $V \in \REAL^{d \times d}$ such that for every $x \in \REAL^d$,
\begin{equation}
\label{eq:rho_svd}
\norm{DV^Tx}_2^p \leq \norm{Ax}_p^p \leq d^{\frac{p}{2}} \norm{DV^Tx}_2^p
\end{equation}
\end{restatable}

\begin{proof}
First, let $L = \br{\Tilde{x} \middle| \Tilde{x} \in \REAL^d, \norm{A\Tilde{x}}_p^p \leq 1}$, and put $x \in \REAL^d$. Observe that
\begin{enumerate}[label=(\roman*)]
    \item since $p \geq 1$, the term $\norm{A\Tilde{x}}_p$ is a convex function for every $\Tilde{x} \in \REAL^d$ which follows from properties of norm function. This means that the level set $L$ is a convex set. \label{prop1}
    \item In addition, by definition of $L$, it holds that for every $\Tilde{x} \in L$, also $-\Tilde{x} \in L$, which makes $L$ a centrally symmetric set by definition.\label{prop2}
    \item Since $A$ is of full rank, then $L$ spans $\REAL^d$.\label{prop3}
\end{enumerate}

Since properties~\ref{prop1}--\ref{prop3} hold, we get by Theorem~\ref{thm:lowner} that there exists a diagonal matrix $D \in [0,\infty)^{d \times d}$ of full rank and an orthogonal matrix $V \in \REAL^d$ such that the set $E = \br{\Tilde{x} \middle| \Tilde{x} \in \REAL^d, \Tilde{x}^T VD^TDV^T \Tilde{x} \leq 1}$ satisfies
\begin{equation}
\label{eq:lownerRes}
\frac{1}{\sqrt{d}} E \subseteq L \subseteq E.
\end{equation}


\paragraph{Proving the right hand side of~\eqref{eq:rho_svd}:} Let $y = \frac{1}{\norm{DV^Tx}_2} x$, and observe that by~\eqref{eq:lownerRes}
\begin{equation}
\label{eq:inclusion_1}
\frac{1}{\sqrt{d}} y \in L.
\end{equation}

Hence,
\begin{equation}
\label{eq:lowner_upper_bound}
\begin{split}
\norm{Ax}_p^p &= \term{\sqrt{d} \norm{DV^Tx}_2}^p \norm{\frac{1}{\sqrt{d}}Ay}_p^p\\
&\leq d^{\frac{p}{2}} \norm{DV^Tx}_2^p,
\end{split}
\end{equation}
where the equality follows from the definition of $y$, and the inequality holds by~\eqref{eq:inclusion_1}.

\paragraph{Proving the left hand side of~\eqref{eq:rho_svd}:} Since $L$ spans $\REAL^d$, then there exists $b > 0$ such that $\norm{A\term{bx}}_p^p = 1$. By~\eqref{eq:lownerRes}, $bx \in E$ which results in $\norm{DV^Tx}_2 = \frac{1}{b}\norm{DV^T \term{bx}}_2 \leq \frac{1}{b}$. Thus,
\begin{equation}
\label{eq:lowner_lower_bound}
\norm{Ax}_p^p = \frac{1}{b^p} \norm{A\term{bx}}_p^p = \frac{1}{b^p} \geq  \norm{DV^Tx}_2^p
\end{equation}

Since all of the above holds for every $x \in \REAL^d$, Lemma~\ref{lem:rho_svd} follows by combining~\eqref{eq:lowner_upper_bound} with~\eqref{eq:lowner_lower_bound}.
\end{proof}

\additiveapprox*

\begin{proof}
First, we assume that $p \neq 2$, otherwise the $\norm{\cdot}_2$ factorization is the \emph{SVD} factorization, and we obtain the optimal solution for the $\ell_2$ low rank approximation problem.
For every $i \in [d]$, let $e_i \in \REAL^d$ be a vector of zeros, except for its $i$th entry where it is set to $1$.

\begin{equation*}
\begin{split}
\norm{A - A_k}_{p, p}^p &= \sum\limits_{i=1}^d \norm{\term{A-A_k}e_i}_p^p \\
&= \sum\limits_{i=1}^d \norm{A \term{I - \term{DV^T}^{-1}D_kV^T}e_i}_p^p
\end{split}
\end{equation*}
where the first equality holds by definition of $\norm{\cdot}_{p,p}^p$, the second equality follows from definition of $A_k$ (see Lines~\ref{algl1:l3}-\ref{algl1:l4} of Algorithm~\ref{algo:l1lowrank}).

Plugging in $A:= A$, $D:= D$, $V := V$, $x:= \term{I - \term{DV^T}^{-1}D_kV^T}e_i$ into Lemma~\ref{lem:rho_svd}, yields that for every $i \in [d]$,
\begin{equation}
\label{eq:additive_approx_for_single_i}
\begin{split}
\norm{\term{D - D_k}V^Te_i}_2^p &\leq \norm{\term{A-A_k}e_i}_{p,p}^p \\
&\leq d^{\frac{p}{2}} \norm{\term{D - D_k}V^Te_i}_2^p.
\end{split}
\end{equation}

Observe that for every $i \in [d]$,
\begin{equation}
\label{eq:addive_approx_1}
\begin{split}
\norm{\term{D - D_k}V^Te_i}_2^p &\leq \norm{D - D_k}_2^p \norm{V^Te_i}_2^p \\
&\leq \norm{D -D_k}_2^p,
\end{split}
\end{equation}
where the first inequality holds by properties of the $\ell_2$ matrix induced norm, and the second inequality holds since $V$ is an orthogonal matrix.

Since $V^Te_i$ is a unit vector,
\begin{equation}
\label{eq:lower_bound_singular}
\norm{\term{D-D_k}V^Te_i}_2^p \geq \sigma_d^p,
\end{equation}
where the inequality holds by plugging $x := V^Te_i$ and $D := \term{D-D_k}$ into Claim~\ref{clm:lowest_singular_value}.

In addition, we have that
\begin{equation}
\label{eq:prop_D}
\sigma_d \leq \norm{D-D_k}_2 = \sigma_{k+1}
\end{equation}
where both the inequality and equality hold since $\sigma_d$ is the lowest eigenvalue of $D$ and $D$ being a diagonal matrix.

By combining~\eqref{eq:additive_approx_for_single_i},~\eqref{eq:addive_approx_1},~\eqref{eq:lower_bound_singular} and~\eqref{eq:prop_D}, we obtain that for every $i \in [d]$,
\begin{equation}
\label{eq:additive_almost_done}
\sigma_d^p \leq \norm{\term{A - A_k}e_i}_p^p \leq d^{\frac{p}{2}} \sigma_{k+1}^p.
\end{equation}

Theorem~\ref{thm:additive} follows by summing~\eqref{eq:additive_almost_done} over every $i \in [d]$.
\end{proof}

\subsection{Proof of Theorem~\ref{coro:additive}}

\begin{theorem}[Variant of Theorem 10~\cite{clarkson2016fast}]
\label{thm:clarkson}
For any $A \in \REAL^{n \times d}$ of rank $d$, one can compute an invertible matrix $R \in \REAL^{d \times d}$ and a matrix $U = A R^{-1}$ such that
\[
\norm{Rx}_2 \leq \norm{Ax}_p \leq d \term{d^3 + d^2 \log{n}}^{\abs{1/p - 1/2}} \norm{Rx}_2,
\]
holds with probability at least $1 - \frac{1}{n}$, where $R$ can be computed in time $O\term{nd\log{n}}$.
\end{theorem}

\additiveapproxcor*
\begin{proof}
The algorithm is described throughout the following proof. Let $R \in \REAL^{d \times d}$ be as defined in Theorem~\ref{thm:clarkson} when plugging $A:= A$ into Theorem~\ref{thm:clarkson}. Let $R = \Tilde{U}DV^T$ be the \emph{SVD} of $R$, $D_k \in [0 ,\infty)^{d \times d}$ be a diagonal matrix where its first $k$ diagonal entries are identical to those of $D$ while the rest of the entries in $D_k$ are set to $0$, and let $\br{\sigma_1,\cdots,\sigma_d}$ be the set of singular values of $D$.

Note that since for every $x \in \REAL^d$, by Theorem~\ref{thm:clarkson} it holds that
\[
\norm{Rx}_2^p \leq \norm{Ax}_p^p d^p \term{d^3 + d^2 \log{n}}^{\abs{1 - p/2}} \norm{Rx}_2^p .
\]

From here, similarly to the proof of Theorem~\ref{thm:additive}, we obtain that
\[
d \sigma_{d}^p \leq \norm{A - A_k}_{p,p}^p \leq d^{1 + p} \term{d^3 + d^2 \log{n}}^{\abs{1 - p/2}} \sigma_{k+1}^p.
\]
\end{proof}

\end{document}